\documentclass[twoside]{article}
\usepackage[accepted]{arxivedoublecolumn}
\pdfoutput=1
\usepackage[hyperindex,breaklinks]{hyperref}

\usepackage{natbib}
\usepackage{url}

\usepackage{amsmath,amssymb}
\usepackage[disable]{todonotes}
\usepackage{graphicx}
\usepackage{subfigure}
 \usepackage{bm}
 
\let\hat\widehat
\let\tilde\widetilde

\newtheorem{theorem}{Theorem}
\newtheorem{lemma}[theorem]{Lemma}

\newenvironment{proof}{{\bf Proof.}}{$\Box$}

\renewcommand{\P}{\mbox{$\mathbb{P}$}}

\newcommand{\E}{\mbox{$\mathbb{E}$}}

 \allowdisplaybreaks

\begin{document}

\twocolumn[

\aistatstitle{Distribution-Free Distribution Regression}

\aistatsauthor{ Barnab\'as P\'oczos \And Alessandro Rinaldo \And Aarti Singh \And Larry Wasserman }

\aistatsaddress{
Carnegie Mellon University\\
Pittsburgh, PA\\
USA, 15213
} ]

\begin{abstract}

`Distribution regression' refers
to the situation where a response $Y$ depends on a covariate $P$
where $P$ is a probability distribution.
The model is $Y=f(P) + \mu$ where $f$ is an unknown regression function and
$\mu$ is a random error.
Typically, we do not observe $P$ directly, but rather, we observe
a sample from $P$.
In this paper we develop theory and methods for
distribution-free versions of distribution regression.
This means that we do not make distributional assumptions
about the error term $\mu$ and covariate $P$.
We prove that when the effective dimension is small enough
(as measured by the doubling dimension), then the excess prediction risk converges to zero with a polynomial rate.

\end{abstract}

\section{Introduction}

In a standard regression model,
we need to predict a real-valued response $Y$
from a vector-valued covariate (or feature)
$X\in\mathbb{R}^d$.
Recently, there has been interest
in extensions of standard regression from finite dimensional Euclidean spaces to other domains.
For example, in functional regression
(\cite{Ferraty2006}) the covariate is a function instead
of a finite dimensional vector.

In this paper, we study
{\em distribution regression}
where the covariate is a probability distribution $P$.
This differs from functional regression in two important ways.
First, $P$ is a probability measure on $\mathbb{R}^k$
rather than a one-dimensional function.
Second, and more importantly, we do not observe the covariate $P$ directly.
Rather, we observe a sample from $P$, which means
that we have a
regression model with measurement error
(\cite{carroll2006measurement}, \cite{fan1993nonparametric}).


\begin{figure}
\begin{center}
\includegraphics[scale=0.3]{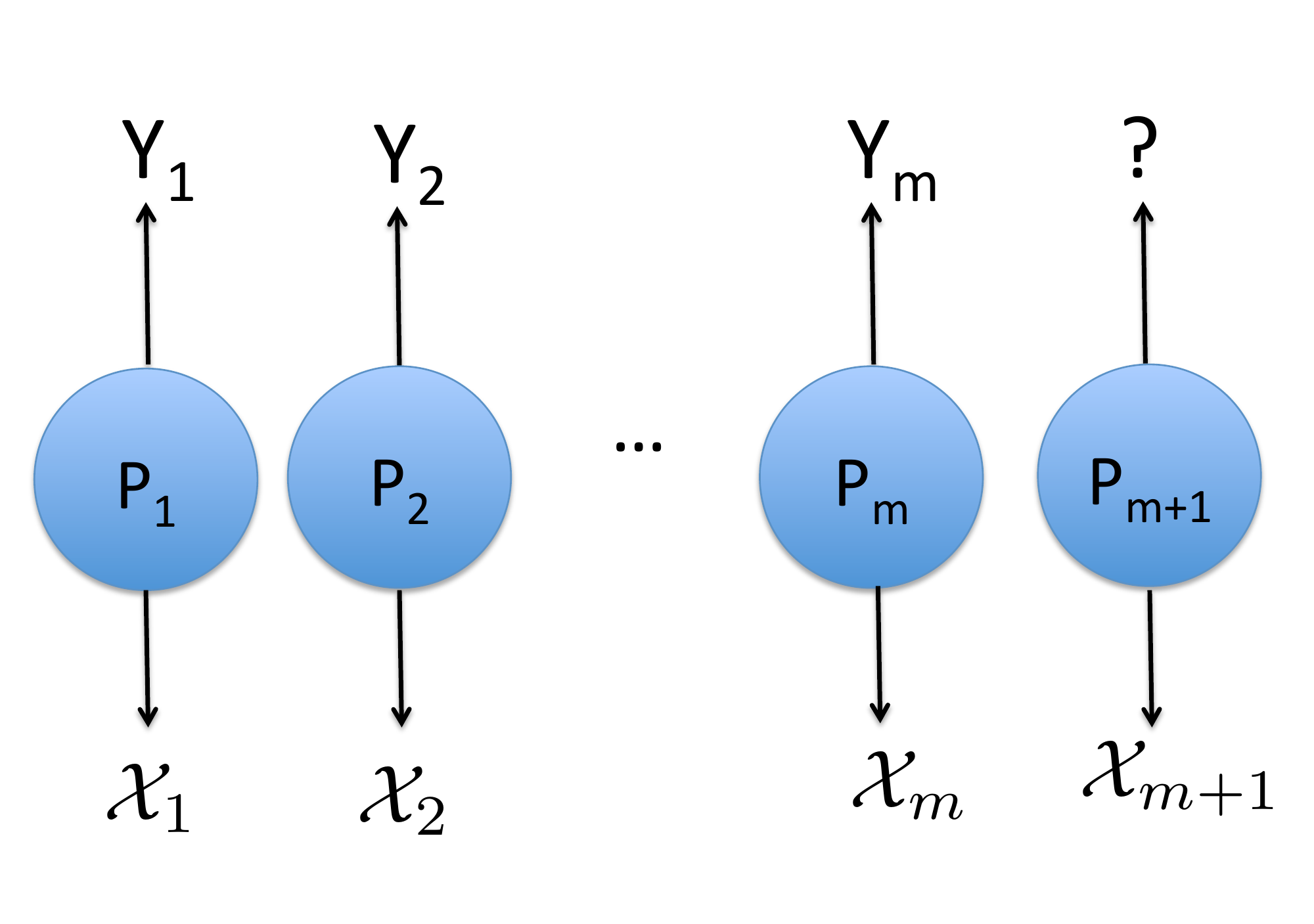}
\end{center}
\caption{Illustration of the model - distributions $P_1, \dots, P_m, P_{m+1}$ are
unobserved, only the ${\cal X}_1,\ldots,{\cal X}_m, {\cal X}_{m+1}$ sample sets are observable.}
\label{fig:DFDR}
\end{figure}

The formal definition of the problem is as follows.
We consider a regression problem with variables
$(P_1,Y_1), \ldots, (P_m,Y_m)$
where
$Y_i \in \mathbb{R}$ and each $P_i$ is a probability distribution
on a compact subset ${\cal K}\subset \mathbb{R}^k$.
We assume that
$$
Y_i = f(P_i) + \mu_i, \quad i =1,\ldots,m,
$$
for some functional $f$, where $\mu_i$
is a noise variable with mean 0.
We do not observe $P_i$ directly;
rather we observe a sample
\begin{equation}
X_{i1},\ldots, X_{i n_i} \stackrel{i.i.d}{\sim} P_i.
\end{equation}
Thus the observed data are
\begin{equation}
({\cal X}_1,Y_1), \ldots, ({\cal X}_m,Y_m)
\end{equation}
where
${\cal X}_i = \{ X_{i1},\ldots, X_{i n_i}\}$.
Our goal is to predict a new $Y_{m+1}$
from a new batch ${\cal X}_{m+1}$ drawn from a new distribution $P_{m+1}$. This model is illustrated in Figure~\ref{fig:DFDR}.

We model the unobservable probability distributions $P_1,\ldots,P_m$ as follows.
Let $\mathbb{D}$ denote the set of all distributions on ${\cal K}$
that have a density with respect to the Lebesgue measure.
We assume that the distributions $P_i$ are an i.i.d. sample from
a measure ${\cal P}$ on $\mathbb{D}$, that is,
$$
P_1,\ldots, P_m, P_{m+1} \stackrel{i.i.d}{\sim} {\cal P}.
$$
Note that $f:\mathbb{D}\to \mathbb{R}$.
If $Q(\cdot|P)$ denotes the law of $Y$ given $P$, then
the joint distribution of $(Y,P)$ is given by
\begin{align*}
\mathbb{P}(Y\in A, P\in B) &= Q(Y\in A| P\in B)\mathcal{P}(P\in B)
\end{align*}


\emph{Our main result} is a theorem where we prove that when the effective dimension measured by the doubling dimension is small enough,
then the estimator is consistent and the prediction risk converges to zero with a polynomial rate.

Our results are \emph{distribution free} in the sense that
the only distributional assumptions we make in this regression problem are that
$\mu_i$ has mean 0 and that
$\mathbb{P}(|Y_i| \leq B_Y)=1$ for some $B_Y$.
We make no other distributional assumptions.

{\textbf{Outline}.}
In Section \ref{section::related}
we discuss related work.
We propose a specific estimator for distribution regression in  Section \ref{section::kernel}. We call
this {\em kernel-kernel estimator} since it makes use of
kernels in two different ways.
In Section \ref{section::upper} we derive an upper bound
on the risk of the estimator. The proofs can be found in Section \ref{section::proofs}.
In Section \ref{section::doubling}
we analyze the risk bound in terms of
the doubling dimension, which is a measure
of the intrinsic dimension of the space.
We present numerical illustrations in Section \ref{section::numerical}.
Finally, we give some concluding remarks in
Section \ref{section::discussion}.

\section{Related work}
\label{section::related}

Our framework is related to functional data analysis, which is a new and steadily improving field of statistics. For comprehensive reviews and references, see \citet{Ramsay05FDA,Ferraty2006}.

A popular approach to do machine learning, such as classification and regression, on the domain of distributions is to embed the distribution to a Hilbert space, introduce kernels between the distributions, and then use a traditional kernel machine to solve the learning problem. There are both parametric and nonparametric methods proposed in the literature.

Parametric methods, (e.g. \citet{Jebara04probabilityproduct,Moreno04akullback-leibler,Jaakkola98exploitinggenerative}), usually fit a parametric family (e.g.\ Gaussians distributions or exponential family) to the densities, and using the fitted parameters they estimate the inner products between the distributions. The problem with parametric approaches, however, is that when the true densities do not belong to the assumed parametric families, then this method introduces some unavoidable bias during the estimation of the inner products between the densities.

A couple of nonparametric approaches exist as well. Since our covariates are represented by finite sets, reproducing kernel Hilbert space (RKHS) based set kernels can be used in these learning problems. \citet{Smola07ahilbert} proposed to embed the distributions to an RKHS using the mean map kernels.
In this framework, the role of universal kernels have been studied
by \citet{christmann:universal}. Recently, the representer
theorem has also been generalized  for the space of probability distributions
\citep{muandet:measure-machines}. \cite{Kondor03akernel} introduced Bhattacharyya's measure of affinity between finite-dimensional Gaussians in a Hilbert space.
In contrast to the previous approaches, \citet{sdm-cvpr,poczos11uai} used nonparametric R\'enyi divergence estimators to solve machine learning problems on the set of distributions.

Although, there are a few algorithms designed for regression on distributions, we know very little about their theoretical properties.
To the best of our knowledge, even the simplest, fundamental questions have not been studied yet. For example, we do not know how many distributions ($m$) and how many samples ($n_i$, $i=1,\ldots,m$) we need to achieve small prediction error.
Our paper is providing an answer to this question.

\section{The Kernel-Kernel Estimator}
\label{section::kernel}

In this section we define an estimator $\hat f$ for the unknown
function $f$.  Our predictor for $Y_{m+1}$ is then $\hat Y_{m+1} =
\hat f({\cal X}_{m+1})$.  Let $\hat P_i$ denote an estimator of $P_i$
based on ${\cal X}_i$, and let  ${\cal X}$ be a sample from a new
distribution $P=P_{m+1}$. Accordingly, we denote
with $\hat P$ an estimator of $P$ based on ${\cal X}$.

Given a bandwidth $h > 0$ and a kernel function $K$ (whose properties will be specified later), we define
\begin{align*}
\hat f (\hat P) &=\hat f (\hat P; \hat P_1,\ldots,\hat P_m)\\
&=\begin{cases}\frac{\sum_i Y_i K\left( \frac{D(\hat P_i,\hat P)}{h}\right)}{\sum_i K\left( \frac{D(\hat P_i,\hat P)}{h}\right)}
&\mbox{if } \sum_i K\left( \frac{D(\hat P_i,\hat P)}{h}\right)>0\\
0 &\mbox{otherwise.}
\end{cases}
\end{align*}

To complete the definition, we need to specify
$\hat P_i$, $\hat P$ and $D$.
We will estimate $P_i$ --- or, more precisely, the density $p_i$ of $P_i$ ---
with a kernel density estimator
\begin{equation}
\hat p_i (x) = \frac{1}{n_i}\sum_{j=1}^{n_i}\frac{1}{b_i^k}
B\left(\frac{\|x-X_{ij}\|}{b_i}\right)
\end{equation}
where $B$ is an appropriate kernel function (see, e.g. \cite{tsybakov.2010}) with bandwidth $b_i > 0$. Here $\|x\|$ denotes the Euclidean norm of $x\in\mathbb{R}^k$.
Accordingly, $\hat P_i$ is defined by
$$
\hat P_i(A) = \int_A \hat p_i(u) du,
$$
for all Borel measurable subsets of $\mathbb{R}^k$.
For any two probabilities in $P$ and $Q$ in $\mathbb{D}$, we take
$D(P,Q) $ to be the $L_1$ distance of their densities: $D(P,Q)= \|p-q\|  = \int |p(x)-q(x)| dx$.
Hence,
\begin{equation}
\hat f(\hat P) = \hat f(\hat P;\hat P_1,\ldots, \hat P_m )=
\frac{\sum_{i=1}^m Y_i K\left( \frac{ ||\hat p - \hat p_i||}{h}\right)}
{\sum_{i=1}^m  K\left( \frac{ ||\hat p - \hat p_i||}{h}\right)}
\end{equation}
which we call the `kernel-kernel estimator' since it makes
use of two kernels, $B$ and $K$.

For simplicity, $n$ will denote the size of the sample $\mathcal{X}$,
and $b$ will be the bandwidth in the estimator of $\hat p$.

In what follows we will make the following assumptions on $f$, $K$, $\mathcal{P}$, $\mu_i$, and $Y_i$.

{\bf Assumptions}

\begin{itemize}
\item {\it (A1) H\"{o}lder continuous functional}.
The unknown functional $f$ belongs to the class ${\cal M}={\cal M}(L,\beta,D)$
of H\"{o}lder continuous functionals on $\mathbb{D}$:
\begin{align*}
{\cal M}  = \Biggl\{f:\  |f(P_i) - f(P_j)| \leq L \, D(P_i,P_j)^{\beta} \Biggr\},
\end{align*}
for some $L>0$ and $0<\beta\leq 1$, where  $D$ is  the above specified $L_1$ metric on $\mathbb{D}$.
In the $\beta=1$ special case this means that $f$ is Lipschitz continuous.

\item {\it (A2) Asymmetric boxed and Lipschitz kernel}. The kernel $K$
satisfies the following properties: $K:[0,\infty]\to \mathbb{R}$ is non-negative and
Lipschitz continuous with Lipschitz constant $L_K$. In addition, there exist
constants $0<\underline{K}<1$ and $0<r<R<\infty$ such that, for all
$x >0$, it holds that
$$
\underline{K}I_{\{x\in \mathcal{B}(0,r)\}}\leq K(x)\leq I_{\{x\in \mathcal{B}(0,R)\}}.
$$

\item {\it (A3) H\"{o}lder class of distributions}.
The distribution ${\cal P}$ is supported
on the set of distributions ${\cal H}_k(1)$
with densities that are $1$-smooth H\"{o}lder functions,
as defined in \cite{rigollet2009optimal}.

\item {\it (A4) Bounded regression}.
We will assume that
${\sup}_{P \in \mathcal{P} }|f(P)| < f_{\max}$
for some $f_{\max} > 0$.
Also, $\mu_i$ has mean $0$ and $\mathbb{P}(|Y_i| \leq B_Y) = 1$ for some $B_Y<\infty$.

\item {\it (A5) Lower bound on $\min_{1\leq i\leq m+1}{n_i}$.}
Let $n=\min_{1\leq i\leq m+1}{n_i}$. We assume that
$e^{n^{\frac{k}{2+k}}}/m \to \infty$
as $m\to\infty$.

\item {\it (A6) Relationship between $n$ and $h$.}
Assume that
$C_* n^{-\frac{1}{2+k}} \leq rh/4$
where $C_*$ is defined in (\ref{eq::star}).

\end{itemize}

\section{Upper Bound on Risk}
\label{section::upper}

We are concerned with upper bounding the risk
$$
R(m,n)=\E \bigg[ |\hat f(\hat P;\hat{P}_1,\ldots,\hat{P}_m )- f(P)| \bigg],
$$
where the expectation is with respect to the joint distribution of
the sample $({\cal X}_1,Y_1),\ldots,({\cal X}_{m},Y_{m})$,
the new covariate $P = P_{m+1}$ and the new observation ${\cal X}_{m+1} $.
Note that the absolute prediction risk is
$\mathbb{E}|\hat Y - Y| \leq R(m,n) + c$,
where $c=\mathbb{E}(|\mu|)$ is a constant.
So bounding the prediction risk is equivalent to bounding
$R(m,n)$, which we call the excess prediction risk.
In what follows,
$C,c_1,c_2,\ldots$
represent constants whose value can be different
in different expressions.

Let $\mathcal{B}(P,h)=\{\tilde P \in \mathbb{D} : D(\tilde P,P)\leq h\}$
denote the $L_1$ ball of distributions around $P$ with radius $h$.
We will see that the risk depends on
the size of the class of probabilities $\mathbb{D}$.
In particular,
the risk depends on
the {\em small ball probability}
$$
\Phi_P(h) = \mathcal{P} (\mathcal{B}(P,h)),
$$
where $P$ is a fixed distribution and $\Phi_P(h)$ is a function of $P$.


Our first result, Theorem~\ref{thm:upper.bound}, provides a general upper bound on the risk. In our second result (Section~\ref{section::doubling}) we show that when the effective dimension measured by the doubling dimension is small, then the risk converges to zero. We also derive an upper bound on the rate of convergence.

\begin{theorem}\label{thm:upper.bound}
Suppose that the assumptions stated above hold.
Let $b=n^{-\frac{1}{2+k}}$ be the bandwidth in the density estimators $\hat p_i$.
Then
\begin{align*}
R&(m,n)\leq
\frac{1}{h}\E\left[\frac1{\Phi_P(rh/2)}\right]
 C_1n^{-\frac{1}{2+k}} +
C_2 h^\beta\\
&+ C_3 \sqrt{\frac{1}{m}} \sqrt{\E\left[\frac1{\Phi_P(rh/2)}\right]}
+ {\frac{C_4}{m}} {\E\left[\frac1{\Phi_P(rh/2)}\right]}
\\
&+(m+1) e^{- \frac{1}{2}n^{\frac{k}{2+k}}},
\end{align*}
where the constants $C_i$'s are specified in the proof.
\end{theorem}

\section{Proof of Theorem~\ref{thm:upper.bound}}\label{section::proofs}

In this Section we prove our main result, Theorem~\ref{thm:upper.bound}. The main idea of the proof is to use the triangle inequality to write
{\small
\begin{align}
R(m,n)& = \E|\hat f(\hat P; \hat P_1,\ldots, \hat P_m )- f(P)| \nonumber\\
&\leq \E|\hat f(\hat P; \hat P_1,\ldots, \hat P_m )-\hat f( P; P_1,\ldots, P_m )| \label{eq:term1}\\
& \qquad+\E|\hat f( P; P_1,\ldots, P_m )- f(P)|. \label{eq:term2}
\end{align}
}
In Sections~\ref{sec:term1} and~\ref{sec:term2} we will derive upper bounds for \eqref{eq:term1} and \eqref{eq:term2}, respectively. Section  \ref{sec:technical}  contains a series of technical results needed in our proofs.


Throughout, we let $\hat K_i=K\left( \frac{D(\hat P_i,\hat P)}{h}\right)$, $K_i=K\left( \frac{D( P_i, P)}{h}\right)$ and $\epsilon_i = K_i - \hat{K}_i$, for $i=1,\ldots,m$. Note that, for ease of readability, we have omitted the dependence on $h$.

\subsection{Technical Results}\label{sec:technical}



\subsubsection{$L_1$ Risk of Density Estimators}\label{sec:MISEdensity}

In this section we bound $\E[D(P,\hat P)|P] = \E[\int | p-\hat p| |P]$,
the $L_1$ risk  of the density estimator $\hat p$ of $p$, uniformly over all $P$ in $\mathbb{D}$.
To this end, suppose that
$n_i \geq n$ for all $i=1,2,\ldots,m+1$, and
let $b_i = b = n^{-\frac{1}{k+2}}$. In this case, the following lemma provides upper bound on the $L_1$ risk of the density estimator.
\begin{lemma}
\begin{align}
\E[D(\hat P_i,P_i)|P_i]&\leq \bar C n^{-\frac{1}{2+k}}, \label{eq:exp_D_hatP_P}\\
\E[D(\hat P_i,P_i)]&\leq \bar C n^{-\frac{1}{2+k}}, \nonumber
\end{align}
where
\begin{equation}\label{eq:hatc}
\bar C =c_0(c_1+c_2),
\end{equation}
with $c_0$, $c_1$ and $c_1$ constants specified in the proof.
\end{lemma}

\begin{proof}
Recall that we assume that
${\cal P}$ is supported
on the set ${\cal H}_k(1)$
of distributions, which are $1$-smooth
$k$-dimensional densities
as defined in
\cite{rigollet2009optimal}.

Let $\E \bigg[ D_2(\hat{P}_i,P)|P_i \bigg] = \E \bigg[ \sqrt{\int (\hat p_i - p_i)^2} \bigg]$ denote the integrated mean squared risk for the density estimator $\hat p_i$ of a fixed density $p_i$.
It then follows from Lemma 4.1 of
\cite{rigollet2009optimal} that
(with an appropriate kernel function $B$),
\begin{align*}
\E[D^2_2(\hat P_i,P_i)|P_i]& \leq
c_1^2 b_i^{2  } + \left(\frac{c_2^2}{n_i b_i^k}\right)
\end{align*}
for some constants $c_1,c_2>0$.

From Jensen's inequality, we have that
$E[X]\leq (\E[X^2])^{1/2}$ for any $X$ random variable. We also know that $(a+b)^{1/2}\leq a^{1/2}+b^{1/2}$ for any $a,b>0$, therefore
\begin{align}
\E[D_2(\hat P_i,P_i)|P_i]&\leq
\left(c_1^2 b_i^2 + \left(\frac{c_2^2}{n_i b_i^k}\right)\right)^{1/2} \nonumber\\
&\leq c_1b_i + \frac{c_2}{n_i^{1/2} b_i^{k/2}}. \nonumber
\end{align}

Since the distributions in $\mathbb{D}$ are supported on a compact set and the kernel $B$ has also compact support, we have, for an appropriate constant $c_0 > 0$,
\[
\int |p_i - \hat p_i | \leq c_0 \sqrt{\int (p_i -\hat p_i)^2}.
\]

Therefore,
\begin{align*}
\E[D(\hat P_i,P_i)|P_i]&\leq c_0 \E[D_2(\hat P_i,P_i)|P_i]\\
&\leq c_0 (c_1b_i + \frac{c_2}{n_i^{1/2} b_i^{k/2}})\\
&\leq c_0(c_1+c_2) n^{-\frac{1}{2+k}},
\end{align*}
where the last step follows from our assumptions that $n_i^{-1/2} b_i^{-k/2} \leq n^{-\frac{1}{2}}n^{\frac{k}{2(k+2)}}=n^{-\frac{1}{k+2}}$, and thus
$$
c_1b_i + \frac{c_2}{n_i^{1/2} b_i^{k/2}} \leq (c_1+c_2) n^{-\frac{1}{2+k}}.
$$
\end{proof}

Next, we show that the terms $D(\hat P_i, P_i)$ are uniformly bounded by a term of order $O(h)$, with high probability.

\begin{lemma}\label{lem:Omega}
With probability no smaller than $1- (m+1) e^{- \frac{1}{2}n^{\frac{k}{2+k}} }$, $D(\hat P_i, P_i) < \frac{rh}{4}$ for all $i =1,\ldots,m+1$.
\end{lemma}

Notice  that by Assumption (A5), $1- (m+1) e^{- \frac{1}{2}n^{\frac{k}{2+k}} } \rightarrow 1$.

\begin{proof}
 From McDiarmid's inequality, for any $\epsilon>0$ we have that
$$
\mathbb{P}( ||\hat p_i - p_i||_1 -  \mathbb{E}||\hat p_i - p_i||_1> \epsilon) \leq  e^{-n\epsilon^2/2}
$$
(see, for example,
section 2.4 of \cite{devroye2001combinatorial}).
Thus,
\begin{align*}
\mathbb{P}( ||\hat p_i - p_i||_1 >\mathbb{E}||\hat p_i - p_i||_1+n^{-\frac{1}{2+k}})& \leq e^{-\frac{1}{2} n^{\frac{k}{2+k}} },
\end{align*}
since $nn^{-\frac{2}{2+k}}=n^{\frac{k}{2+k}}$. This implies that
\begin{align*}
\mathbb{P}( \max_{1\leq i \leq m+1} ||\hat p_i - p_i||_1 >\mathbb{E}||\hat p_i - p_i||_1+n^{-\frac{1}{2+k}}))\\
\leq (m+1) e^{- \frac{1}{2}n^{\frac{k}{2+k}} }\to 0,
\end{align*}
by assumption (A5).
Therefore,
\begin{align*}
1&-(m+1) e^{- \frac{1}{2}n^{\frac{k}{2+k}} } \\
&\leq \mathbb{P}( \max_{1\leq i \leq m+1} ||\hat p_i - p_i||_1 \leq \mathbb{E}||\hat p_i - p_i||_1+n^{-\frac{1}{2+k}}))\\
&\leq \mathbb{P}( \max_{1\leq i \leq m+1} ||\hat p_i - p_i||_1 \leq (1+c_0(c_1+c_2))n^{-\frac{1}{2+k}}).
\end{align*}
This implies that with
\begin{align}\label{eq::star}
C_*=(1+c_0(c_1+c_2))
\end{align}
and using assumption (A6),
we have that
\begin{equation}\label{eq::bigsave}
D(\hat P_i,P_i) \leq C_* n^{-\frac{1}{k+2}}\leq \frac{rh}{4}\ \ \ {\rm for\ all\ }i
\end{equation}
on an event $\Omega_{m,n}$,
where
$\mathbb{P}(\Omega_{m,n}^c) \leq (m+1) e^{- \frac{1}{2}n^{\frac{k}{2+k}} }$. Here
$\Omega_{m,n}^c$ denotes the complement of $\Omega_{m,n}$.
\end{proof}

\subsubsection{Other Lemmata}

Throughout this section we will make use of the constant $\bar C$, defined in \eqref{eq:hatc}.
In what follows, we will need a few lemmas that we list below. Their proofs can be found in the supplementary material.

The following lemma provides an upper bound on
$\mathbb{P}(\sum_{i=1}^m K_i=0)$ with the help of small ball probabilities.
\begin{lemma}
\label{lem:smallball}
\begin{align*}
\mathbb{P}\Bigl(\sum_{i=1}^m K_i=0\Bigr) &\leq
\mathbb{P}\Bigl(\sum_{i=1}^m K_i<\underbar{K}\Bigr)
=\frac{1}{e m}\mathbb{E}\left[\frac{1}{\Phi_{P}(rh)}\right].
\end{align*}
\end{lemma}
We will also need the following lemma.
\begin{lemma}\label{lemma:sum.over.K1}
\[
\E \left[\frac1{\sum_i K_i} I_{\{\sum_i K_i \geq\underline{K}\}}\right]
\leq \frac{1+1/\underline{K}}{m\underline{K}} \E \bigg[ \frac{1}{\Phi_P(Rh)} \bigg].
\]
\end{lemma}

The following lemma provides an upper bound on $|\epsilon_i|$. 
\begin{lemma}\label{lemma:bounding_epsilon}
Assume that the kernel function $K$ is Lipschitz continuous with Lipschitz constant $L_K$.
We have that
\begin{align*}
|\epsilon_i| &\leq \frac{L_K}{h}({D(P,\hat P)}+{D( P_i, \hat P_i)}).
\end{align*}
\end{lemma}

By definition, $|\epsilon_i|=|K_i -\hat K_i|=|K(\frac{D(P,P_i)}{h})-K(\frac{D( \hat P, \hat P_i)}{h})|$,
which is a deterministic function of random variables $P$, $P_i$, $\hat P$, and $\hat P_i$. We will denote this deterministic relationship as $\epsilon_i=\epsilon_i(P,\hat P, P_i, \hat P_i)$. The following
lemma shows that for any $\kappa>0$,
$$
\mathbb{P}\Bigl(\sum_i |\epsilon_i(P,\hat P, P_i, \hat P_i)|<\kappa| \{P_i\}_{i=1}^m, P\Bigr)
$$
can be lower bounded by a non-trivial quantity that does not depend on $P$ and $\{P_i\}_{i=1}^m$.
\begin{lemma}\label{lem:prob.sum.epsilon}
For any $\kappa>0$ we have that
\begin{align*}
\mathbb{P}&(\sum_i |\epsilon_i(P,\hat P, P_i, \hat P_i)| <\kappa|  \{P_i\}_{i=1}^m, P)\geq \eta,
\end{align*}
where $\eta=\eta(\kappa,n,m) = 1-\frac{2 L_K m\bar{C}}{h\kappa} n^{-\frac{1}{2+k}}$.
\end{lemma}

The following lemma provides an upper bound on the expected value of $\sum_{i=1}^m|\epsilon_i|$.

\begin{lemma}\label{lemma:sum.epsilon.j}
\begin{align*}
 \E&\left[ \sum_{i=1}^m |\epsilon_i| \bigg| P,\{P_i\}_{i=1}^m \right]\leq\frac{ 2L_K\bar{C}m}{h} n^{-\frac{1}{2+k}}.
\end{align*}
\end{lemma}

The next lemma shows that
$\mathbb{P}\Bigl(\sum_{i=1}^m \hat K_i<\underbar{K}\Bigr)$ can be upper bounded by a small quantity as well.
We assume that
$n_i = n$ and $b_i = b$ for all $i$.
Define
\begin{align*}
\zeta &= \zeta(n,m)=
\frac{1}{em}\mathbb{E}\left(\frac{1}{\Phi_P\left(\frac{rh}{2}\right)}\right)+
(m+1) e^{- \frac{1}{2}n^{\frac{k}{2+k}} }.
\end{align*}

\begin{lemma}\label{lem:zeta}
$$
\mathbb{P}\Bigl(\sum_{i=1}^m \hat K_i=0\Bigr) \leq \mathbb{P}\Bigl(\sum_{i=1}^m \hat K_i<\underbar{K}\Bigr) \leq \zeta.
$$
\end{lemma}

\subsection{Upper bound on Equation~\ref{eq:term1}} \label{sec:term1}


Let $\Delta \hat f= |\hat f(\hat P; \hat P_1,\ldots, \hat P_m )-\hat f( P; P_1,\ldots, P_m )|$.
Our goal is to provide an upper bound on $\E[\Delta \hat f]$.

Introduce the following events: $E_0=\{\sum_i K_i=0\}$, $E_1=\{0<\sum_i K_i<\underline{K} \}$, $E_2=\{\underline{K} \leq\sum_i K_i\}$.
Similarly, $\hat E_0=\{\sum_i \hat K_i=0\}$, $\hat  E_1=\{0<\sum_i \hat  K_i<\underline{K} \}$,
$\hat E_2=\{\underline{K} \leq \sum_i \hat  K_i\}$.
Obviously,
$\E[\Delta \hat f]=\sum_{k=0}^2\sum_{l=0}^2 \E[\Delta \hat f I_{E_k}I_{\hat E_l}].$

Based on the sign of $\sum_i K_i$ and $\sum \hat K_i$, there are four different cases.
(i) If $\sum_i K_i>0$ and $\sum_i \hat K_i>0$, then $\Delta \hat f=|\frac{\sum_i Y_i \hat K_i}{\sum_i \hat K_i}-\frac{\sum_i Y_i K_i}{\sum_i K_i}|
$. (ii)
If $\sum_i K_i>0$ and $\sum_i \hat K_i=0$, then $
\Delta \hat f=|\frac{\sum_i Y_i K_i}{\sum_i K_i}|$. (iii)
If $\sum_i K_i=0$ and $\sum_i \hat K_i>0$, then $
\Delta \hat f=|\frac{\sum_i Y_i \hat K_i}{\sum_i \hat K_i}|$, and finally (iv) if
$\sum_i K_i=0$ and $\sum_i \hat K_i=0$, then $\Delta \hat f=0$. From this it immediately follows that $\E[\Delta \hat f I_{E_0}I_{\hat E_0}]=0$.

When $\sum_i K_i>0$,
$\left|\sum_i \frac{ Y_i K_i}{\sum_i K_i}\right|\leq B_Y$.
Therefore,
\begin{align*}
\E&\left[\left|\sum_i  \frac{  Y_i K_i}{\sum_i K_i}\right| I_{\hat E_0} (I_{E_1}+I_{E_2})\right]\\
 &\qquad\leq B_Y\E\left[I_{\{\sum_i K_i>0 \wedge \sum_i \hat K_i=0\}}\right]\\
 &\qquad=B_Y \mathbb{P}(\sum_i K_i>0 , \sum_i \hat K_i=0)\\
&\qquad\leq B_Y\mathbb{P}(\sum_{i=1}^m \hat K_i=0)\leq  B_Y\zeta(n,m).
 \end{align*}

Similarly,
\begin{align*}
\E&\left[\left|\sum_i \frac{  Y_i \hat K_i}{\sum_i \hat K_i}\right|
 I_{E_0} (I_{\hat E_1}+I_{\hat E_2}) \right]\leq  \frac{B_Y}{e m}\int\frac{d {\mathcal{P}}( P)}{{\Phi}_{ P}(rh)}.
 \end{align*}
It is also easy to see that
\begin{align*}
\E&\left[\Delta \hat f I_{E_1} (I_{\hat E_1}+I_{\hat E_2})  \right]\\
&\leq\E\bigg[\hspace*{-1mm}\left(\left|\sum_i \frac{  Y_i K_i}{\sum_i K_i}\right|\hspace*{-1mm}+\hspace*{-1mm}\left|\sum_i \frac{  Y_i \hat K_i}{\sum_i \hat K_i}\right|\right)  I_{E_1} (I_{\hat E_1}+I_{\hat E_2})  \bigg]\\
&\leq\E\bigg[2B_Y  I_{E_1} (I_{\hat E_1}+I_{\hat E_2})  \bigg]\leq 2B_Y\E\bigg[  I_{E_1} \bigg]\\
&= 2B_Y \mathbb{P}(\sum_{i=1}^m 0<  K_i<\underline{K}/2)
\leq \frac{2B_Y}{em} \int \frac{d \mathcal{P}( P)} {{\Phi}_{ P}(rh)}.
\end{align*}
Similarly,
\begin{align*}
\E\left[\Delta \hat f I_{\hat E_1} (I_{  E_1}+I_{  E_2})  \right]
&\leq 2B_Y \mathbb{P}(\sum_{i=1}^m 0<  \hat K_i<\underline{K}/2)\\
&\leq 2B_Y \zeta(n,m).
\end{align*}

All that left is to upper bound $\E\left[\Delta \hat f I_{ E_2} I_{ \hat E_2}  \right]$. The next lemma provides an upper bound for this.
\begin{lemma}\label{lem:DeltaE2E2}
\begin{align*}
\E&\left[\Delta \hat f I_{ E_2} I_{ \hat E_2}  \right]\leq C_1 \frac{1}{h} \E \bigg[\frac{1}{\Phi_P(Rh)} \bigg]  n^{-\frac{1}{2+k}}.
\end{align*}
\end{lemma}
The proof can be found in the supplementary material.

Finally, putting the pieces together we obtain the following theorem.
\begin{theorem}\label{thm:firstterm}
\begin{align*}
\E&|\hat f(\hat P; \hat P_1,\ldots, \hat P_m )-\hat f( P; P_1,\ldots, P_m )|  \\
& \leq C_1 \frac{1}{h} \E \bigg[\frac{1}{\Phi_P(rh/2)} \bigg]  n^{-\frac{1}{2+k}} +C_2 \frac{1}{m}\mathbb{E}\left[\frac{1}{ \Phi_{P}(rh/2)}\right]\\
&\qquad+(m+1) e^{- \frac{1}{2}n^{\frac{k}{2+k}}}  .
\end{align*}
\end{theorem}
The proof can be found in the supplementary material.


\subsection{Upper bound on Equation~\ref{eq:term2}} \label{sec:term2}

In this section we show that under the above specified conditions
$\E|\hat f( P; P_1,\ldots, P_m )- f(P)|$ can be upper bounded by
\begin{align*}
 C_1(h^\beta)+C_2\left(\sqrt{\E\left[\frac{1}{m \Phi_P(rh/2)}\right]}\right) +{\frac{C_3}{m}\E\left[\frac{1}{ \Phi_{P}(rh/2)}  \right]},
\end{align*}
where the expectation is with respect to the random probability measure $P$ in $\mathcal{P}$.

We have to bound
$\E|\hat f( P; P_1,\ldots, P_m )- f(P)|$.
Note that $Y_i=f(P_i)+\mu_i$, and
\begin{align*}
\E|\hat f&( P; P_1,\ldots, P_m )- f(P)| \\
& = \E\left|\frac{\sum_i Y_i  K_i}{\sum_i K_i}I_{\{\sum_i K_i >0\}} - f(P)\right| \\
&=\E\left|\frac{\sum_i (f(P_i)+\mu_i)  K_i}{\sum_i K_i}I_{\{\sum_i K_i >0\}} - f(P)\right| \\
&\leq \E\left[\left|\frac{\sum_i (f(P_i)-f(P))  K_i}{\sum_i K_i} + \frac{\sum_i \mu_i  K_i}{\sum_i K_i}\right|I_{\{\sum_i K_i >0\}}\right] \\
& \quad + \E\left[\left|f(P)\right|I_{\{\sum_i K_i = 0\}}\right|\\
& \le
\E\left[\frac{\sum_i |f(P_i)-f(P)|  K_i}{\sum_i K_i}I_{\{\sum_i K_i > 0\}}\right] \\
& \quad + \E\left[\left|\frac{\sum_i \mu_i  K_i}{\sum_i K_i}\right|I_{\{\sum_i K_i > 0\}}\right]  + f_{\max}\mathbb{P}(\sum_i K_i = 0).
\end{align*}

We will bound each of the three terms next. For the first term, since $f$ is H\"{o}lder-$\beta$
we have
\begin{align*}
\E&\left[\frac{\sum_i |f(P_i)-f(P)|  K_i}{\sum_i K_i}I_{\{\sum_i K_i > 0\}}\right]\\
& \le  \E\left[\frac{\sum_i LD(P_i,P)^{\beta}  K_i}{\sum_i K_i}I_{\{\sum_i K_i > 0\}}\right]
\leq L\, (hR)^{\beta},
\end{align*}
where in the last step we used the fact that
$$
D(P_i,P)^{\beta}  K_i = D(P_i,P)^{\beta} K\left(\frac{D(P_i,P)}{h}\right)\leq (hR)^{\beta}K_i,
$$
since $supp(K)\subseteq B(0,R)$.

We now bound the second term.
{
\begin{align*}
\E&\left[\left|\frac{\sum_i \mu_i  K_i}{\sum_i K_i}\right|I_{\{\sum_i K_i > 0\}}\right]\\
& = \E\left[\left|\frac{\sum_i \mu_i  K_i}{\sum_i K_i}\right|I_{\{\sum_i K_i \geq \underline{K}\}}
+ \left|\frac{\sum_i \mu_i  K_i}{\sum_i K_i}\right|I_{\{\underline{K}>\sum_i K_i > 0\}}\right]\\
& \leq \E\left[\left|\frac{\sum_i \mu_i  K_i}{\sum_i K_i}\right|I_{\{\sum_i K_i \geq \underline{K}\}}\right]
+ B_Y\P(\underline{K}>\sum_i K_i)\\
& \leq \E\left[\left|\frac{\sum_i \mu_i  K_i}{\sum_i K_i}\right|I_{\{\sum_i K_i \geq \underline{K}\}}\right]
+ \frac{B_Y}{em}\int\frac{d\mathcal{P}(P)}{\Phi_{P}(rh)}.
\end{align*}
}
(A4) implies that $\mathbb{P}(|\mu_i|\leq B_Y)=1$, i.e. $B_Y$ is a bound on the noise. The last step follows from Lemma~\ref{lem:smallball}.
For the first term in the above expression, we use the following lemma. Its proof can be found in the supplementary material.
\begin{lemma}\label{lem:part_in_2ndterm}
\begin{align*}
\E\left[\left|\frac{\sum_i \mu_i  K_i}{\sum_i K_i}\right|I_{\{\sum_i K_i \geq \underline{K}\}}\right]\leq B_Y\sqrt{\frac{1+1/\underline{K}}{m\underline{K}}\ \int\frac{d\mathcal{P}(P)}{\Phi_P(Rh) }}.
\end{align*}
\end{lemma}

Finally, we bound the third term using Lemma~\ref{lem:smallball}:
$$
f_{\max}\mathbb{P}(\sum_i K_i = 0) \leq  \frac{f_{\max}}{e m}\int\frac{d\mathcal{P}(P)}{\Phi_{P}(rh)}.
$$

Putting everything together, we have
\begin{align*}
\E&|\hat f( P; P_1,\ldots, P_m )- f(P)|\\& \leq L(hR)^\beta
 + B_Y\sqrt{\frac{1+1/\underline{K}}{m\underline{K}}\ \int\frac{d\mathcal{P}(P)}{\Phi_P(Rh) }} \\
 &\quad +\frac{B_Y}{em}\int\frac{d\mathcal{P}(P)}{\Phi_{P}(rh)} + \frac{f_{\max}}{e m}\int\frac{d\mathcal{P}(P)}{\Phi_{P}(rh)}\\
& \leq C_1 h^\beta + C_2\sqrt{\frac{1}{m}\E\left[\frac{1}{ \Phi_{P}(rh/2)}  \right]}+ {\frac{C_3}{m}\E\left[\frac{1}{ \Phi_{P}(rh/2)}  \right]}.
\end{align*}
Note that $\Phi_{P}(rh/2)\leq \Phi_{P}(rh) \leq \Phi_{P}(Rh)$.


\section{Doubling Dimension}
\label{section::doubling}

The upper bound on the risk in Theorem~\ref{thm:upper.bound} depends on the quantity
$
\E \bigg[  \frac{1}{\Phi_P(rh/2)}\bigg].
$
In future work, we will show that, without further assumptions,
this quantity can be quite large which leads to very slow rates of
convergence.
This is because the covering number of the class ${\cal H}_k(1)$ is huge.
For this paper, we concentrate on the more optimistic case where
the support of ${\cal P}$ has
small effective dimension.

One way to measure effective dimension is to use the doubling dimension.
Following \cite{kpotufe2011k},
we say that ${\cal P}$ is a doubling measure with effective dimension $d$ if,
for every $r>0$ and $0 < \epsilon < 1$,
\begin{equation}
\frac{{\cal P}({\cal B}(s,r))}{{\cal P}({\cal B}(s,\epsilon r))} \leq
\left(\frac{c}{\epsilon}\right)^d.
\end{equation}

If $d$ denotes the doubling dimension of measure $\mathcal{P}$, then the
$\sqrt{  \E[ 1/(m \Phi_P(rh/2))] }$ term
in Theorem~\ref{thm:upper.bound}
can be upper bounded as follows:
\begin{align*}
 &\sqrt{  \E \bigg[ \frac{1}{m \Phi_P(rh/2) } }  \bigg] = \sqrt{   \E \bigg[ \frac{1}{m} \frac{\Phi_P(1) }{ \Phi_P(rh/2) } \frac{1}{ \Phi_P(1) }  \bigg] }\\
& \leq \sqrt{ \frac{1}{m} C (rh/2)^{-d} \E \bigg[ \frac{1}{ \Phi_P(1) } \bigg] } \leq \frac{C}{\sqrt{m h^{d}}}.
\end{align*}

Note also that when $mh^{d}\leq 1$, then $\frac{1}{mh^{d}}\leq \frac{1}{\sqrt{mh^{d}}}$.
In this case, as a corollary of Theorem \ref{thm:upper.bound}, we now have that
\begin{equation}\label{eq:risk.O}
R(m,n)  \leq
\frac{C_1}{ h^{d+1}n^{1/(k+2)} } +
C_2 h^{\beta} + C_3 \sqrt{\frac{1}{m h^{d}}},\\
\end{equation}
for appropriate constants $C_1$, $C_2$ and $C_3$.
%

To derive the rates for the risk, we consider two separate cases, depending on whether the
third term in the right hand side of \eqref{eq:risk.O} dominates the first term or not.

Thus first assume that
\begin{equation}\label{eq:first.case}
\sqrt{\frac{1}{m h^{d}}} = \Omega \left( \frac{C_1}{ h^{d+1}n^{1/(k+2)} } \right),
  \end{equation}
  so that the risk becomes, asymptotically, $O \left(  h^{\beta} +  \sqrt{\frac{1}{m h^{d}}}\right)$.
The optimal choice for $h$ is then $ \Theta \left(m^{-1/(2\beta+d)}\right)$, yielding a rate for the risk
\[
R(m,n) = O\left(m^{-\beta/(2\beta+d)}\right).
  \]
  Notice that this choice of $h$ ensures that our assumption (A6) is met, since in this case \eqref{eq:first.case} implies that
\[
  n = \Omega\left( m^{\frac{\beta+d+1}{2 \beta+d}(k+2)}\right),
  \]
from which we obtain that
\[
  h = \Theta\left( m^{- \frac{1}{2 \beta + d}}\right) = \Omega \left( n^{- \frac{1}{(k+2)(\beta + d + 1)}}\right) = \Omega\left( n^{-\frac{1}{k+2}}\right).
  \]
This rate is reasonable because if the number of samples per distribution
$n$ is large compared to the number $m$ of distributions, then the learning rate is limited by the number of distributions
$m$ and is in fact precisely the same as the
rate of learning a standard $\beta$-H\"{o}lder smooth regression function in
$d$ dimensions. That is, the the effect of not knowing the distributions
$P_1, \dots, P_m$ exactly and only having a finite sample from the distributions
is negligible.

For the second case, suppose that
\begin{equation}\label{eq:second.case}
\sqrt{\frac{1}{m h^{d}}} = O\left( \frac{1}{ h^{d+1}n^{1/(k+2)} } \right).
  \end{equation}
  Then,  $R(m,n) = O \left( \frac{1}{ h^{d+1}n^{1/(k+2)}} + h^{\beta} \right)$, which implies that the optimal choice for $h$  is  $h = \Theta\left( n^{-\frac{1}{(k+2)(\beta+d+1)}}\right)$, giving the rate
  \[
    R(m,n) = O \left( n^{-\frac{\beta}{(k+2)(\beta+d+1)}}\right).
    \]
    Just like before, this choice of $h$ does not violate assumption (A6) since
      \[
  h = \Theta \left( n^{-\frac{1}{(k+2)(\beta+d+1)}}\right)  = \Omega\left( n^{- \frac{1}{k+2}}\right).
  \]
Notice that, \eqref{eq:second.case} also implies that
    \[
      m = \Omega \left( n^{\frac{2\beta + d}{(k+2)(\beta+d+1)}} \right).
      \]
In this case, the rate is limited by the number of samples per distribution $n$, as
expected. Notice that the rate gets worse as the dimensionality of each distribution
$k$ grows and as the smoothness $\beta$ of the regression function deteriorates.

\noindent {\bf Remark.} If there is no additive noise, i.e. $\mu_i = 0$, similar calculations yield that
$
  R(m,n) = O \left( m^{-\frac{1}{\beta+d}}\right)$
when
$
  n = \Omega \left( m^{ \frac{\beta+d+1}{(\beta+d)(k+2)}} \right)
 $,
and $
  R(m,n) = O \left( n^{-\frac{\beta}{(k+2)(\beta+d+1)}}\right)$
  otherwise.
  While the rates seem reasonable, establishing optimality of the rates by demonstrating
matching lower bounds is an open question that we plan to investigate in future work.

\section{Numerical Illustrations}
\label{section::numerical}

The following experiments serve as a proof of concepts to demonstrate the applicability of the distribution regression estimator in Section~\ref{section::kernel}. In these experiments, we used triangle kernels ($k(x)=1-|x|$ if $-1 \leq x \leq 1$, and $0$ otherwise). We set all the $n,n_1,\ldots,n_m$ set sizes and $b,b_1,\ldots, b_m$ bandwidths to the same values, which will be specified below.
In the first experiment, we generated 325 sample sets from $Beta(a,3)$ distributions where $a$ was varied between $[3,20]$ randomly.
We constructed $m=250$ sample sets for training, 25 for validation, and 50 for testing. Each sample set contained $n=500$ $Beta(a,3)$ distributed i.i.d.\ points. Our task in this experiment was to learn the skewness of $Beta(a,b)$ distributions, $
 f=\frac{2(b-a)\sqrt{a+b+1}}{(a+b+2)\sqrt{ab}}$. We considered the noiseless case, i.e. $\mu$ was set to zero. Our estimator of course is not aware of that the sample sets are coming from beta distributions, and it does not know the skewness function values in the test sets either; its values are available only in the training and validation sets.

To find appropriate bandwidths  $b$ and $h$, we sampled 100 i.i.d. randomly and uniformly distributed values in [0,1], evaluated the MSE performance of the distribution regression estimator on the validation test using these bandwidths parameters, and then chose that bandwidth parameters the lead to the best values on the validation test. To estimate the $L_2$ distances between $\hat p_i$ and $\hat p$, we calculated their estimated values in 4096 points on a uniformly distributed grid between the min an max values in the sample sets, and then estimated the integral $\int (p(x)-\hat p_i(x))^2 d(x)$ with the rectangle method numerical integration.
Figure~\ref{fig:beta-skewness} displays the predicted values for the 50 test sample sets, and we also show the true values of the skewness functions. As we can see the true and the estimated values are very close to each other.

In the next experiment, our task was to learn the entropy of Gaussian distributions. We chose a $2\times2$ covariance matrix $\Sigma=AA^T$, where $A\in\mathbb{R}^{2\times 2}$, and $A_{ij}$ was randomly selected from $U[0,1]$. Just as in the previous experiments we constructed 325 sample sets from $\{\mathcal{N}(0,R(\alpha_i)\Sigma^{1/2})\}_{i=1}^{325}$. Where $R(\alpha_i)$ is a 2d rotation matrix with rotation angle $\alpha_i=i\pi/325$. From each $\mathcal{N}(0,R(\alpha_i)\Sigma^{1/2})$ distribution we sampled 500 2-dimensional i.i.d.\ points. Similarly to the previous experiment, 250 points was used for training, 25 for selecting appropriate bandwidth parameters, and 50 for training. Our goal was to learn the entropy of the first marginal distribution: $f=\frac{1}{2}\ln(2\pi e\sigma^2)$, where $\sigma^2=M_{1,1}$ and $M=R(\alpha_i)\Sigma R^T(\alpha_i)\in\mathbb{R}^{2 \times 2}$. $\mu$ was zero in this experiment as well. Figure~\ref{fig:gauss_entropy} displays the learned entropies of the 50 test sample sets. The true and the estimated values are close to each other in this experiment as well.
\begin{figure}  [h] 
  \centering
 \hbox{\subfigure[Skewness of Beta]{
\includegraphics[width = 1.6in]{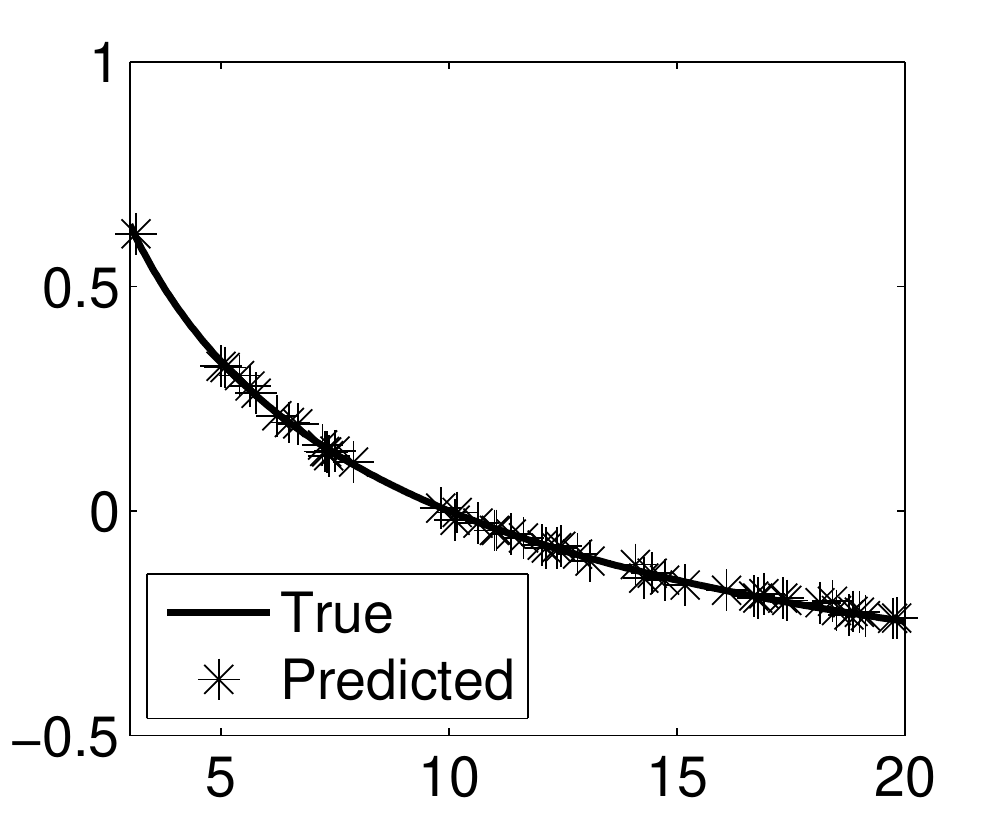} \label{fig:beta-skewness}
}
\hspace*{-.5cm}
  \subfigure[Entropy of Gaussian]{
\includegraphics[width = 1.6in]{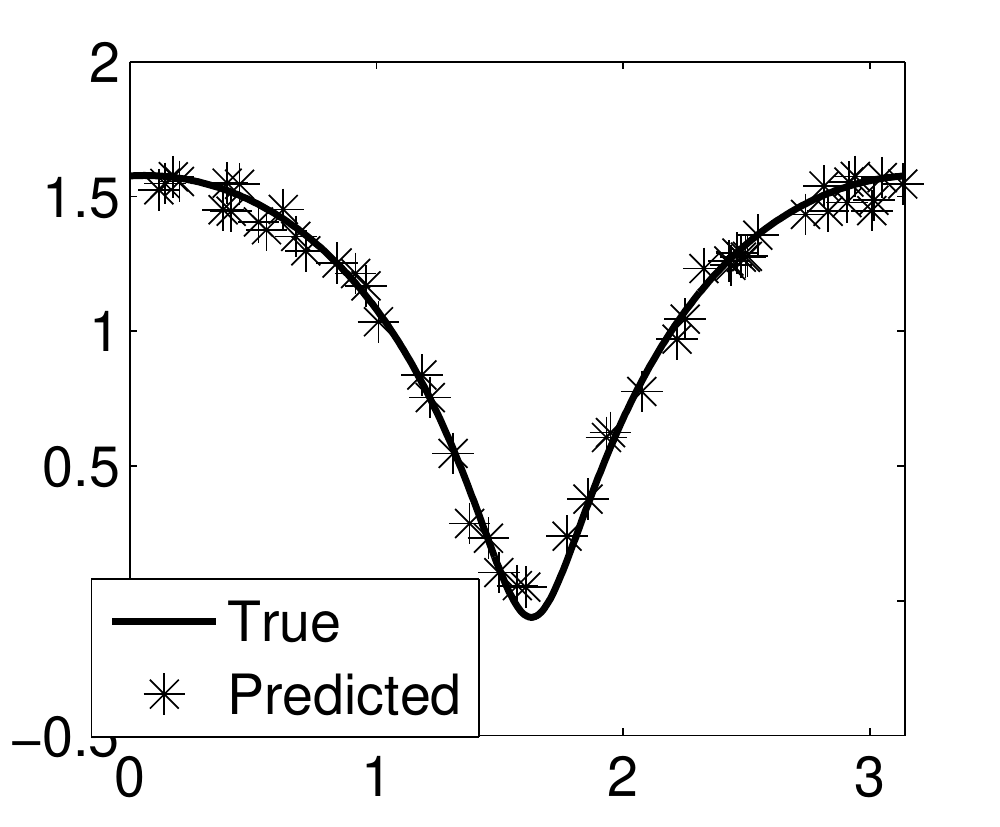} \label{fig:gauss_entropy}}}
 \vspace*{-1mm}
\caption{(a) Learned skewness of $Beta(a,3)$ distribution. Axis $x$: parameter $a$ in $[3,20]$. Axis $y$: skewness of $Beta(a,3)$. (b) Learned entropy of a 1d marginal distribution of a rotated 2d Gaussian distribution. Axes $x$: rotation angle in $[0,\pi]$. Axis $y$: entropy.}
\end{figure}

\section{Discussion and Conclusion}
\label{section::discussion}

We have presented an estimator for
distribution regression
which is distribution-free in the sense that
the estimator makes no strong distributional assumptions
on the error variables.
We derived upper bounds on the risk of the estimator
and, in particular, we analyzed the case with a finite
doubling dimension.

We note that our rates are faster than the logarithmic rates
that are sometimes
obtained in measurement error nonparametric regression models as
in \cite{fan1993nonparametric}.
The reason is that the logarithmic rates
occur when the measurement error is Gaussian.
Our measurement error
corresponds to $||\hat p_i - p_i||$ 
which is not Gaussian for finite $n_i$ and which decreases
when $n_i$ increases.
 In the standard measurement error model,
the error is $O(1)$ and is not decreasing.

In future work,
we will prove lower bounds
which show that, without further assumptions
(such as assumptions about the doubling dimension),
the rates can be very slow.
Also, we will show that
similar results hold for other estimators
such as $k$-nn estimators
and RKHS estimators.

 \bibliography{biblio}
\bibliographystyle{plainnat}


\newpage
\onecolumn
\newpage
\twocolumn
\section*{Supplementary material}

\subsection*{Proof of Lemma~\ref{lem:smallball}}
 \begin{proof}
The proof follows the argument of \cite{gyorfi2002}.
\begin{align*}
\mathbb{P}\Bigl(\sum_{i=1}^m K_i<\underbar{K}\Bigr)
&=\mathbb{P}\left(\sum_{i=1}^m K\left(\frac{D(P_i,P)}{h}\right)<\underbar{K}\right)\\
&\leq\mathbb{P}\left(\sum_{i=1}^m I_{\{D(P_i,P)\geq rh\}}=0\right),
\end{align*}
since according to our assumptions on kernel $K$ if for some $i$ it holds that $D(P_i,P)/h \leq r$, then $K_i\geq $ \underbar{K}.
Therefore,
\begin{align}
\mathbb{P}&\Bigl(\sum_{i=1}^m K_i<\underbar{K}\Bigr) \nonumber
\leq\mathbb{P}\left(\sum_{i=1}^m I_{\{D(P_i,P)\geq rh\}}=0\right) \nonumber\\
&=\E[\mathbb{P}(\sum_{i=1}^m I_{\{D(P_i,P)\geq rh\}}=0\,| P )] \nonumber\\
&= \int \mathbb{P}\left(\sum_{i=1}^m I_{\{D(P_i,P)\geq rh\}}=0\, \bigg| P \right)d\mathcal{P}(P) \nonumber\\
&=\int[1-\mathcal{P}(P_1\in \mathcal{B}(P,rh)|P)]^m   d\mathcal{P}(P) \label{eq:_tmp_P_i_iid}\\
&\leq\int \exp[-m\mathcal{P}(P_1\in \mathcal{B}(P,rh)|P)]    d\mathcal{P}(P) \label{eq:_tmp_1muexp}\\
&=\int \exp[-m\mathcal{P}(P_1\in \mathcal{B}(P,rh)|P)] \nonumber\\
&\qquad \times \frac{m\mathcal{P}(P_1\in \mathcal{B}(P,rh)|P)}{m\mathcal{P}(P_1\in \mathcal{B}(P,rh)|P)}    d\mathcal{P}(P) \nonumber\\
&\leq  \max_{u>0} u \exp(-u)\int\frac{d\mathcal{P}(P)}{m\mathcal{P}(P_1\in \mathcal{B}(P,rh)|P)} \label{eq:maxuexp} \\
&\leq \frac{1}{e}\int\frac{d\mathcal{P}(P)}{m\mathcal{P}(P_1\in \mathcal{B}(P,rh)|P)}
=\frac{1}{e m}\mathbb{E}\left[\frac{1}{\Phi_{P}(rh)}\right], \nonumber
\end{align}
where we used in \eqref{eq:_tmp_P_i_iid}, \eqref{eq:_tmp_1muexp}, and \eqref{eq:maxuexp} respectively that $\{P_i\}$ are iid, $(1-u)^m\leq \exp(-um)$ for all $0\leq u\leq 1$, $m\geq 1$, and $\max(u\exp(-u))=\frac{1}{e}$.
\end{proof}

\subsection*{Proof of Lemma~\ref{lemma:sum.over.K1}}

\begin{proof}
\begin{align*}
\E&\left[\frac1{\sum_i K_i} I_{\{\sum_i K_i \geq\underline{K}\}}\right]
\leq \E\left[\frac{1+1/\underline{K}}{1+\sum_i K_i}\right] \\
& \leq  \E\left[\frac{1+1/\underline{K}}{1+\underline{K}\sum_i I_{\{D(P_i,P)\leq hR\}}}\right] \\
& = \frac{1+1/\underline{K}}{\underline{K}}\ \E\left[\frac{1}{1/\underline{K}+ \sum_i I_{\{D(P_i,P)\leq hR\}}}\right] \\
& \leq \frac{1+1/\underline{K}}{\underline{K}}\ \E  \left[\frac{1}{1+ \sum_i I_{\{D(P_i,P)\leq hR\}}}  \right] \\
& = \frac{1+1/\underline{K}}{\underline{K}}\ \E \bigg[ \left[\frac{1}{1+ \sum_i I_{\{D(P_i,P)\leq hR\}}} \big| P \right] \bigg]\\
& \leq \frac{1+1/\underline{K}}{m\underline{K}}\E \bigg[ \frac{1}{\Phi_P(Rh)} \bigg],
\end{align*}
where the second-to-last line uses the fact that $\underline{K}<1$ and the last line follows
since for a binomial random variable $B(m,p)$, $\E[\frac1{1+B(m,p)}] \leq \frac1{(m+1)p} \leq \frac1{mp}$.
\end{proof}

\subsection*{Proof of Lemma~\ref{lemma:bounding_epsilon}}

  \begin{proof}
 $D(P,Q)$ is a distance, therefore the triangle inequality holds, and we have that
\begin{align*}
|\epsilon_i|=|K_i -\hat K_i|&=\left|K(\frac{D(P,P_i)}{h})-K(\frac{D( \hat P, \hat P_i)}{h})\right|\\
&\leq \frac{L_K}{h}|{D(P,P_i)}-{D( \hat P, \hat P_i)}|\\
&\leq \frac{L_K}{h}({D(P,\hat P)}+{D( P_i, \hat P_i)}).
\end{align*}
Here we used that
\begin{align*}
D&(P,P_i)-D( \hat P, \hat P_i)\\
&\leq [D(P,\hat P)+D(\hat P,\hat P_i)+D(\hat P_i, P_i)]-D( \hat P, \hat P_i)\\
&=D(P,\hat P)+D(\hat P_i, P_i),
\end{align*}
and
\begin{align*}
D&( \hat P, \hat P_i)-D(P,P_i)\\
&\leq [D(\hat P,P)+D(P,P_i)+D( P_i, \hat P_i)]-D( P, P_i)\\
&=D(\hat P,P)+D(P_i, \hat P_i).
\end{align*}
\end{proof}

\subsection*{Proof of Lemma~\ref{lem:prob.sum.epsilon}}

\begin{proof}
From Markov's inequality, for any $X$, $Y$ and constant $\kappa>0$,
\begin{align*}
1 \leq \frac{\E[|X|\, |Y]}{\kappa}+\mathbb{P}(|X|< \kappa|Y).
\end{align*}
Thus,
\begin{align}
\mathbb{P}&\Bigl(\sum_i |\epsilon_i|<\kappa\Bigm| \{P_i\}_{i=1}^m, P\Bigr)\\\
&\geq 1-\frac{\E[\sum_i |\epsilon_i| | \{P_i\}_{i=1}^m,P]}{\kappa} \nonumber\\
&=1-\frac{\sum_i \E[ |\epsilon_i| | P_i,P]}{\kappa} \nonumber\\
&\geq 1-\frac{L_K}{h\kappa}{\sum_{i=1}^m \E[ ({D(P,\hat P)}+{D( P_i, \hat P_i)})| P_i,P]} \label{eq:Peikappa_1}\\
&\geq 1-\frac{L_K}{h\kappa}{m2\bar{C}n^{-\frac{1}{2+k}}} =\eta(\kappa,n,m) . \nonumber
\end{align}
Here \eqref{eq:Peikappa_1} holds
due to Lemma~\ref{lemma:bounding_epsilon}, and we also used \eqref{eq:exp_D_hatP_P}.
\end{proof}

\subsection*{Proof of Lemma~\ref{lemma:sum.epsilon.j}}

\begin{proof}
The term $\E\left[ \sum_{i=1}^m |\epsilon_i| \bigg| P,\{P_i\}_{i=1}^m \right]$ is upper bounded by
\begin{align*}
  & \frac{L_K}{h}\sum_{i=1}^m \E\left[D(P,\hat P)+D( P_i, \hat P_i) \bigg| P,\{P_i\}_{i=1}^m \right]\\
& \leq \frac{L_K}{ h}  2\bar{C}m n^{-\frac{1}{2+k}}.
\end{align*}
\end{proof}

\subsection*{Proof of Lemma~\ref{lem:zeta}}

\begin{proof}
Recall that
$D(\hat P_i,P_i) \leq rh/4$ for all $i$
on an event $\Omega_{m,n}$ and that
$\mathbb{P}(\Omega_{m,n}^c) \leq (m+1) e^{- \frac{1}{2}n^{\frac{k}{2+k}} }$.
So, on $\Omega_{m,n}$,
\begin{align*}
D(\hat P_i,\hat P) &\leq
D(\hat P_i,P_i) + D(P,\hat P) + D(P_i,P)\\
& \leq D(P_i,P) + \frac{rh}{2}.
\end{align*}
Now, using the event $\Omega_{m,n}$ defined in Lemma \ref{lem:Omega},
\begin{align*}
\mathbb{P}&\Bigl(\sum_{i=1}^m \hat K_i=0\Bigr)\leq \mathbb{P}\Bigl(\sum_{i=1}^m \hat K_i<\underbar{K}\Bigr)\\
 &=
\mathbb{P}\Bigl(\Omega_{m,n},\sum_{i=1}^m \hat K_i<\underbar{K}\Bigr)+ \mathbb{P}\Bigl(\Omega_{m,n}^c,\sum_{i=1}^m \hat K_i <\underbar{K}\Bigr)\\
& \leq
\mathbb{P}\Bigl(\Omega_{m,n},\sum_{i=1}^m \hat K_i<\underbar{K}\Bigr) +
\mathbb{P}\Bigl(\Omega_{m,n}^c\Bigr)\\
& \leq
\mathbb{P}\Bigl(\Omega_{m,n},\sum_{i=1}^m  \hat K_i<\underbar{K}\Bigr) +
(m+1) e^{- \frac{1}{2}n^{\frac{k}{2+k}} }
\end{align*}
and
\begin{align*}
\mathbb{P}\Bigl(\Omega_{m,n},\sum_{i=1}^m  \hat K_i<\underbar{K}\Bigr) &=
\mathbb{P}(\Omega_{m,n},\sum_{i=1}^m I_{D(\hat P_i,\hat P) \geq rh} =0)\\
&\leq
\mathbb{P}(\sum_{i=1}^m I_{D(P_i,P) \geq rh/2} =0)\\
& \leq
\frac{1}{em}\mathbb{E}\left[ \frac{1}{\Phi_P(rh/2)}\right].
\end{align*}
The result follows.
\end{proof}

\subsection*{Proof of Lemma~\ref{lem:DeltaE2E2}}

\begin{proof}
\begin{align*}
\E&\left[\Delta \hat f I_{ E_2} I_{ \hat E_2}  \right]=\E\left[\left|
\frac{\sum_i Y_i \hat K_i}{\sum_j \hat K_j}
-\frac{\sum_i Y_i K_i}{\sum_j K_j}
\right|I_{ E_2} I_{ \hat E_2}\right]\\
&=\E\left[\left|\sum_i Y_i\left(\frac{  K_i}{\sum_j K_j}-\frac{ \hat K_i}{\sum_j \hat K_j}\right)
\right|I_{ E_2} I_{ \hat E_2} \right]\\
&\leq B_Y\E\left[\sum_i \left|\left(\frac{  K_i}{\sum_j K_j}-\frac{ \hat K_i}{\sum_j \hat K_j}\right)
\right|I_{ E_2} I_{ \hat E_2} \right]\\
&= B_Y\E\left[\sum_i \left(\frac{|  K_i (\sum_j \hat K_j)- \hat K_i (\sum_j K_j)|}{(\sum_j \hat K_j) (\sum_j K_j)}\right)
I_{ E_2} I_{ \hat E_2} \right]\\
&= B_Y\E\bigg[\sum_i \left(\frac{|  (\hat K_i - \epsilon_i) (\sum_j \hat K_j)-  \hat K_i (\sum_j ( \hat K_j - \epsilon_j) )|}{(\sum_j \hat K_j) (\sum_j K_j)}\right)\\
&\qquad \qquad \times I_{ E_2} I_{ \hat E_2} \bigg]\\
&=B_Y\E\bigg[\sum_i \left(\frac{|  - \epsilon_i (\sum_j \hat K_j) +  \hat K_i (\sum_j \epsilon_j)|}{(\sum_j \hat K_j) (\sum_j K_j)}\right)
I_{ E_2} I_{ \hat E_2} \bigg]\\
&\leq B_Y\E\bigg[ \bigg(\frac{ (\sum_i \hat K_i) (\sum_j |\epsilon_j|)}{(\sum_j \hat K_j) (\sum_j K_j)} \\
&\quad + \frac{ (\sum_i |\epsilon_i|) (\sum_j \hat K_j)}{(\sum_j \hat K_j) (\sum_j K_j)}\bigg)
I_{ E_2} I_{ \hat E_2} \bigg]\\
&=B_Y\E\bigg[ \bigg(\frac{  \sum_j |\epsilon_j|}{\sum_j   K_j}+ \frac{ (\sum_i |\epsilon_i|)}{\sum_j   K_j}\bigg)
I_{ E_2} I_{ \hat E_2} \bigg]\\
&\leq 2B_Y\E\bigg[ \frac{  \sum_j |\epsilon_j|}{\sum_j  K_j} I_{  E_2} \bigg]\\
&= \E\bigg[ \E \bigg[  \sum_j |\epsilon_j| |P,\{ P_i \}_{i=1}^m \bigg] \frac{1}{\sum_j  K_j} I_{  E_2} \bigg]\\
&  \leq  2B_Y \frac{L_K}{ h}  2\hat{c}m n^{-\frac{1}{2+k}}\E \bigg[  \frac{1}{\sum_j  K_j} I_{  E_2} \bigg]\\
&\leq 2B_Y  \frac{L_K}{ h}  2\hat{c}m n^{-\frac{1}{2+k}} \frac{1+1/\underline{K}}{m\underline{K}} \E \left[\frac{1}{\Phi_P(Rh)}\right]\\
&=C_1 \frac{1}{h} \E \bigg[\frac{1}{\Phi_P(Rh)} \bigg]  n^{-\frac{1}{2+k}}.
\end{align*}
where we used Lemma \ref{lemma:sum.epsilon.j} and Lemma \ref{lemma:sum.over.K1}.
\end{proof}

\subsection*{Proof of Theorem~\ref{thm:firstterm}}

 \begin{proof}
\begin{align*}
\E|\hat f&(\hat P; \hat P_1,\ldots, \hat P_m )-\hat f( P; P_1,\ldots, P_m )|\\
&\leq  \E\left[\Delta \hat f I_{ E_2} I_{ \hat E_2}\right]+ 3B_Y\zeta +
3\frac{B_Y}{e m}   \E \bigg[\frac{1}{\Phi_P(rh)} \bigg]\\
& \leq C_1 \frac{1}{h} \E \bigg[\frac{1}{\Phi_P(Rh)} \bigg]  n^{-\frac{1}{2+k}}
 +C_2 \frac{1}{m}\mathbb{E}\left[\frac{1}{ \Phi_{P}(rh)}\right]\\
&\qquad+ C_3 \frac{1}{m}\mathbb{E}\left[\frac{1}{ \Phi_{P}(rh/2)}\right]+(m+1) e^{- \frac{1}{2}n^{\frac{k}{2+k}}}.
\end{align*}
Note also that $\Phi_{P}(rh/2) \leq \Phi_{P}(rh)\leq \Phi_P(Rh)$.
\end{proof}

\subsection*{Proof of Lemma~\ref{lem:part_in_2ndterm}}
 \begin{proof}
Notice that
if $\sum_i K_i \geq \underline{K} > 0$,
$$
\text{var}\left(\frac{\sum_i \mu_i  K_i}{\sum_i K_i}|P,P_1,\dots, P_m\right)
\leq B_Y^2 \frac{\sum_i K_i^2}{(\sum_i K_i)^2}.
$$

Using this and H\"{o}lder's inequality, we get:
\begin{align*}
& \E\left[\left|\frac{\sum_i \mu_i  K_i}{\sum_i K_i}\right|I_{\{\sum_i K_i \geq \underline{K}\}}\right]\\
& = \E\left[\E\left[\left|\frac{\sum_i \mu_i  K_i}{\sum_i K_i}\right|I_{\sum_i K_i \geq \underline{K}}|P,P_1,\dots, P_m\right]\right]\\
& \leq \E\left[\sqrt{\text{var}\left(\frac{\sum_i \mu_i  K_i}{\sum_i K_i}|P,P_1,\dots, P_m\right)}I_{\{\sum_i K_i \geq \underline{K}\}}\right]\\
& \leq \E\left[B_Y\frac{\sqrt{\sum_i K_i^2}}{\sum_i K_i}I_{\{\sum_i K_i \geq \underline{K}\}}\right]\\
&\leq \E\left[B_Y\frac{\sqrt{\sum_i K_i}}{\sum_i K_i}I_{\{\sum_i K_i \geq \underline{K}\}}\right] \\
& \leq B_Y \sqrt{\E\left[\frac{\sum_i K_i}{(\sum_i K_i)^2}I_{\{\sum_i K_i \geq \underline{K}\}}\right]}\\
& \leq B_Y\sqrt{\E\left[\frac{I_{\{\sum_i K_i \geq \underline{K}\}}}{\sum_i K_i}\right]} \\
& \leq B_Y\sqrt{\frac{1+1/\underline{K}}{m\underline{K}}\ \int\frac{d\mathcal{P}(P)}{\Phi_P(Rh) }}.
\end{align*}
The second inequality holds since $K(x) < 1$ and
the last step stems from Lemma \ref{lemma:sum.over.K1}.
\end{proof}

\end{document}